%% file: iclr2020_conference.tex
\documentclass{article} 
\usepackage{iclr2020_conference,times}

\usepackage{graphicx}
\usepackage{booktabs}
\usepackage{multirow}
\usepackage{mathrsfs}
\usepackage{amsmath}
\usepackage{hyperref}
\usepackage{bbm}
\usepackage{amssymb}
\usepackage{wrapfig}
\usepackage{balance}
\usepackage{url}
\usepackage{bm}
\usepackage{tabularx}
\usepackage{float}
\usepackage{amsmath}
\usepackage{amsthm}
\usepackage{amsfonts}
\usepackage{amssymb}
\usepackage{graphicx}
\usepackage{booktabs}
\usepackage{placeins}
\usepackage{multirow}
\usepackage[ruled, linesnumbered]{algorithm2e}
\usepackage{epstopdf}
\usepackage{color}
\usepackage{subcaption}
\usepackage[T1]{fontenc}
\usepackage{amsmath,amsfonts,bm}
\usepackage{balance}
\usepackage{enumitem}
\usepackage{xcolor}
\usepackage{balance}
\usepackage{wrapfig}
\usepackage{relsize}
\usepackage{subcaption}
\usepackage{enumitem}
\usepackage{wrapfig}
\usepackage{graphics}
\usepackage{url}
\usepackage{makecell}
\usepackage[utf8]{inputenc}
\usepackage[english]{babel}
\usepackage{url}
\usepackage{stackengine,amssymb,graphicx,scalerel}
\usepackage{breqn}
\usepackage{amsthm}

\input{math_commands.tex}

\usepackage{hyperref}
\usepackage{url}

\title{Composition-based Multi-Relational \\ Graph Convolutional Networks}

\iclrfinalcopy 

\author{Shikhar Vashishth\thanks{Equally Contributed}~~\thanks{Work done while at IISc, Bangalore}$\ \ ^{1,2}$ \quad Soumya Sanyal$^{*1}$ \quad Vikram Nitin\footnotemark[2]$\ \ ^{3}$ \quad Partha Talukdar$^{1}$\\
	$^{1}$Indian Institute of Science, $^{2}$Carnegie Mellon University, $^{3}$Columbia University\\
	\small \texttt{svashish@cs.cmu.edu}, \small  \texttt{\{shikhar,soumyasanyal,ppt\}@iisc.ac.in}, \\ \small \texttt{vikram.nitin@columbia.edu}
}
\begin{document}

\input{defs}

\maketitle

\begin{abstract}
Graph Convolutional Networks (GCNs) have recently been shown to be quite successful in modeling graph-structured data. However, the primary focus has been on handling simple undirected graphs. Multi-relational graphs are a more general and prevalent form of graphs where each edge has a label and direction associated with it. Most of the existing approaches to handle such graphs suffer from over-parameterization and are restricted to learning representations of nodes only. In this paper, we propose \method{}, a novel Graph Convolutional framework which jointly embeds both nodes and relations in a relational graph. \method{} leverages a variety of entity-relation composition operations from Knowledge Graph Embedding techniques and scales with the number of relations. It also generalizes several of the existing multi-relational GCN methods. We evaluate our proposed method on multiple tasks such as node classification, link prediction, and graph classification, and achieve demonstrably superior results. We make the source code of \method{} available to foster reproducible research.
\end{abstract}

\input{sections/introduction}
\input{sections/related_work}

\input{sections/background}
\input{sections/overview}
\input{sections/details}

\input{sections/experiments}

\input{sections/results}
\input{sections/conclusion}

\section*{Acknowledgments}
We thank the anonymous reviewers for their constructive comments. This work is supported in part by the Ministry of Human Resource Development (Government of India) and Google PhD Fellowship.

\bibliography{references}
\bibliographystyle{iclr2020_conference}

\newpage
\appendix
\section{Appendix}

\input{sections/appendix}


\end{document}

%% file: math_commands.tex
\usepackage{amsmath,amsfonts,bm}

\def\eqref#1{equation~\ref{#1}}

\def\1{\bm{1}}

\DeclareMathAlphabet{\mathsfit}{\encodingdefault}{\sfdefault}{m}{sl}
\SetMathAlphabet{\mathsfit}{bold}{\encodingdefault}{\sfdefault}{bx}{n}



%% file: defs.tex
\newcommand{\refalg}[1]{Algorithm \ref{#1}}
\newcommand{\refeqn}[1]{Equation \ref{#1}}
\newcommand{\reffig}[1]{Figure \ref{#1}}
\newcommand{\reftbl}[1]{Table \ref{#1}}
\newcommand{\refsec}[1]{Section \ref{#1}}

\newcommand{\add}[1]{\textcolor{red}{#1}\typeout{#1}}
\newcommand{\remove}[1]{\sout{#1}\typeout{#1}}

\newcommand{\m}[1]{\mathcal{#1}}
\newcommand{\bmm}[1]{\bm{\mathcal{#1}}}
\newcommand{\real}[1]{\mathbb{R}^{#1}}
\newcommand{\method}{\textsc{CompGCN}}

\newcommand{\bleuone}{BLEU${_1}$\xspace}
\newcommand{\bleu}{BLEU${_4}$\xspace}

\newcommand{\problem}{DD}
\newcommand{\problemfull}{Document Dating}

\newtheorem{theorem}{Theorem}[section]
\newtheorem{claim}[theorem]{Claim}

\newcommand{\reminder}[1]{\textcolor{red}{[[ #1 ]]}\typeout{#1}}
\newcommand{\reminderR}[1]{\textcolor{gray}{[[ #1 ]]}\typeout{#1}}

\newcommand{\tensor}{\mathcal{X}}
\newcommand{\Real}{\mathbb{R}}

\newcommand{\tuples}{\mathbb{T}}

\newcommand\norm[1]{\left\lVert#1\right\rVert}

\newcommand{\note}[1]{\textcolor{blue}{#1}}

\newcommand*{\Scale}[2][4]{\scalebox{#1}{$#2$}}%
\newcommand*{\Resize}[2]{\resizebox{#1}{!}{$#2$}}%

\def\mat#1{\mbox{\bf #1}}

\newcommand{\datafb}{FB15k}
\newcommand{\datawn}{WN18}
\newcommand{\datafbn}{FB15k-237}
\newcommand{\datawnn}{WN18RR}
\newcommand{\datayago}{YAGO3-10}

\theoremstyle{definition}
\newtheorem{definition}{Definition}[section]
 
\theoremstyle{proposition}
\newtheorem{proposition}{Proposition}[section]
\newtheorem*{lemma*}{Lemma}

%% file: sections/introduction.tex
\section{Introduction}
\label{sec:introduction}

Graphs are one of the most expressive data-structures which have been used to model a variety of problems. Traditional neural network architectures like Convolutional Neural Networks \citep{alexnet} and Recurrent Neural Networks \citep{lstm} are constrained to handle only Euclidean data. Recently, Graph Convolutional Networks (GCNs) \citep{Bruna2013,Defferrard2016} have been proposed to address this shortcoming, and have been successfully applied to several domains such as social networks \citep{graphsage}, knowledge graphs \citep{r_gcn}, natural language processing \citep{gcn_srl}, drug discovery \citep{deep_chem}, crystal property prediction \citep{mtcgcnn}, and natural sciences \citep{protein_protein_gcn}.

However, most of the existing research on GCNs \citep{Kipf2016,graphsage,gat} have focused on learning representations of nodes in simple undirected graphs. A more general and pervasive class of graphs are multi-relational graphs\footnote{In this paper, multi-relational graphs refer to graphs with edges that have labels and directions.}. A notable example of such graphs is knowledge graphs. Most of the existing GCN based approaches for handling relational graphs \citep{gcn_srl,r_gcn} suffer from over-parameterization and are limited to learning only node representations.
Hence, such methods are not directly applicable for tasks such as link prediction which require relation embedding vectors.
Initial attempts at learning representations for relations in graphs \citep{dual_primal_gcn,graph2seq} have shown some performance gains on tasks like node classification and neural machine translation.

\begin{figure*}[t]	
	\centering	
	\includegraphics[width=\textwidth]{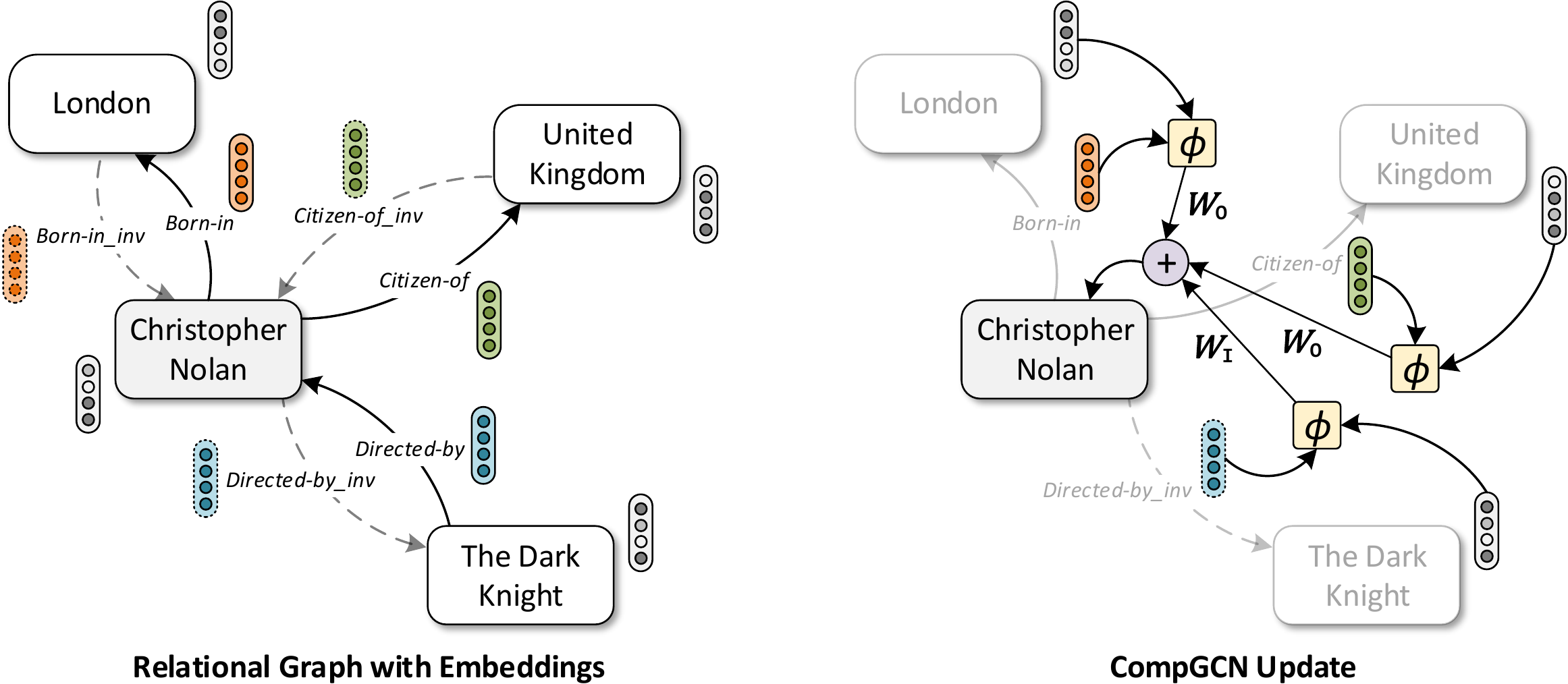}	
	\caption{\label{fig:method_overview} Overview of \method{}. Given node and relation embeddings, \method{} performs a composition operation $\phi(\cdot)$ over each edge in the neighborhood of a central node (e.g. \textit{Christopher Nolan} above). The composed embeddings are then convolved with specific filters $\bm{W}_O$ and $\bm{W}_I$ for original and inverse relations respectively. We omit self-loop in the diagram for clarity. The message from all the neighbors are then aggregated to get an updated embedding of the central node. Also, the relation embeddings are transformed using a separate weight matrix. Please refer to Section \ref{sec:details} for details.
	}	
\end{figure*}

There has been extensive research on embedding Knowledge Graphs (KG) \citep{survey2016nickel,survey2017} where representations of both nodes and relations are jointly learned. These methods are restricted to learning embeddings using link prediction objective. Even though GCNs can learn from task-specific objectives such as classification, their application has been largely restricted to non-relational graph setting. 
Thus, there is a need for a framework which can utilize KG embedding techniques for learning task-specific node and relation embeddings.  In this paper, we propose \method{}, a novel GCN framework for multi-relational graphs which systematically leverages entity-relation composition operations from knowledge graph embedding techniques.
\method{} addresses the shortcomings of previously proposed GCN models by jointly learning vector representations for both nodes and relations in the graph. An overview of \method{} is presented in Figure \ref{fig:method_overview}.
The contributions of our work can be summarized as follows:

\begin{enumerate}[itemsep=2pt,parsep=0pt,partopsep=0pt,leftmargin=*,topsep=0.2pt]
	\item We propose \method{}, a novel framework for incorporating multi-relational information in Graph Convolutional Networks which leverages a variety of composition operations from knowledge graph embedding techniques to jointly embed both nodes and relations in a graph.
	\item We demonstrate that \method{} framework generalizes several existing multi-relational GCN methods (Proposition \ref{prop:reduction}) and also scales with the increase in number of relations in the graph (Section \ref{sec:results_basis}). 
	\item Through extensive experiments on tasks such as node classification, link prediction, and graph classification, we demonstrate the effectiveness of our proposed method.
\end{enumerate} 
\noindent The source code of \method{} and datasets used in the paper have been made available at {\color{blue} \url{http://github.com/malllabiisc/CompGCN}}.

%% file: sections/related_work.tex
\section{Related Work}
\label{sec:related_work}

\textbf{Graph Convolutional Networks:} GCNs generalize Convolutional Neural Networks (CNNs) to non-Euclidean data. GCNs were first introduced by \cite{Bruna2013} and later made scalable through efficient localized filters in the spectral domain \citep{Defferrard2016}. A first-order approximation of GCNs using Chebyshev polynomials has been proposed by \cite{Kipf2016}. Recently, several of its extensions have also been formulated \citep{graphsage,gat,gin}. Most of the existing GCN methods follow \textit{Message Passing Neural Networks} (MPNN) framework \citep{mpnn} for node aggregation. 
Our proposed method can be seen as an instantiation of the MPNN framework. However, it is specialized for relational graphs.



\textbf{GCNs for Multi-Relational Graph:} An extension of GCNs for relational graphs is proposed by \cite{gcn_srl}. However, they only consider direction-specific filters and ignore relations due to over-parameterization. \cite{r_gcn} address this shortcoming by proposing basis and block-diagonal decomposition of relation specific filters. \textit{Weighted Graph Convolutional Network}  \citep{sacn_paper} utilizes learnable relational specific scalar weights during GCN aggregation. While these methods show performance gains on node classification and link prediction, they are limited to embedding only the nodes of the graph. 
Contemporary to our work, \cite{vrgcn} have also proposed an extension of GCNs for embedding both nodes and relations in multi-relational graphs. However, our proposed method is a more generic framework which can leverage any KG composition operator. We compare against their method in Section \ref{sec:results_link}.


\textbf{Knowledge Graph Embedding:} 
Knowledge graph (KG) embedding is a widely studied field \citep{survey2016nickel, survey2017} with application in tasks like link prediction and question answering \citep{kg_question_answering}. Most of KG embedding approaches define a score function and train node and relation embeddings such that valid triples are assigned a higher score than the invalid ones. Based on the type of score function, KG embedding method are classified as translational \citep{transe,transh}, semantic matching based \citep{distmult,hole} and neural network based \citep{ntn_kg,conve}. In our work, we evaluate the performance of \method{} on link prediction with methods of all three types.

%% file: sections/background.tex
\section{Background}
\label{sec:background}

In this section, we give a brief overview of Graph Convolutional Networks (GCNs) for undirected graphs and its extension to directed relational graphs.

\textbf{GCN on Undirected Graphs:}
\label{sec:back_undirected_graph}
Given a graph $\m{G} = (\m{V}, \m{E}, \bm{\m{X}})$, where $\m{V}$ denotes the set of vertices, $\m{E}$ is the set of edges, and $\bm{\m{X}} \in \mathbb{R}^{|\m{V}| \times d_0} $ represents $d_0$-dimensional input features of each node. The node representation obtained from a single GCN layer is defined as:
$ \bm{H} = f\big(\bm{\hat{A}}\bm{\m{X}} \bm{W}\big)$.
Here, $\bm{\hat{A}}=\widetilde{\bm{D}}^{-\frac{1}{2}} (\bm{A} + \bm{I}) \widetilde{\bm{D}}^{-\frac{1}{2}}$ is the normalized adjacency matrix with added self-connections and $\widetilde{\bm{D}}$ is defined as $\widetilde{\bm{D}}_{ii} = \sum_{j}(\bm{A} + \bm{I})_{ij}$. The model parameter is denoted by $\bm{W} \in \mathbb{R}^{d_0 \times d_1}$ and $f$ is some activation function. The GCN representation $\bm{H}$ encodes the immediate neighborhood of each node in the graph. For capturing multi-hop dependencies in the graph, several GCN layers can be stacked, one on the top of another as follows:
$\bm{H}^{k+1} = f\big(\bm{\hat{A}}\bm{H}^{k} \bm{W}^k\big), $
where $k$ denotes the number of layers, $\bm{W}^{k} \in \mathbb{R}^{d_k \times d_{k+1}}$ is layer-specific parameter and $\bm{H}^0 = \bm{\m{X}}$.

\textbf{GCN on Multi-Relational Graphs:}
\label{sec:back_relational_graph}
For a multi-relational graph $\m{G} = (\m{V}, \m{R}, \m{E}, \bm{\m{X}})$, where $\m{R}$ denotes the set of relations, and each edge $(u, v, r)$ represents that the relation $r \in \m{R}$ exist from node $u$ to $v$. The GCN formulation as devised by \citet{gcn_srl} is based on the assumption that information in a directed edge flows along both directions. Hence, for each edge $(u,v,r) \in \m{E}$, an inverse edge $(v,u,r^{-1})$ is included in $\m{G}$. The representations obtained after $k$ layers of directed GCN is given by
\begin{equation}
\label{eq:dir_gcn}
\bm{H}^{k+1} = f\big(\bm{\hat{A}}\bm{H}^{k} \bm{W}_r^k \big)  .
\end{equation}
Here, $\bm{W}_{r}^{k}$ denotes the relation specific parameters of the model. However, the above formulation leads to over-parameterization with an increase in the number of relations and hence, \citet{gcn_srl} use direction-specific weight matrices. \citet{r_gcn} address over-parameterization by proposing basis and block-diagonal decomposition of $\bm{W}_{r}^{k}$.

%% file: sections/overview.tex
%
%

%% file: sections/details.tex
\section{\method{} Details}
\label{sec:details}

In this section, we provide a detailed description of our proposed method, \method{}. The overall architecture is shown in Figure \ref{fig:method_overview}. We represent a multi-relational graph by $\m{G}=(\m{V}, \m{R}, \m{E}, \bm{\m{X}},  \bm{\m{Z}})$ as defined in Section \ref{sec:back_relational_graph} where $\bm{\m{Z}} \in \real{|\m{R}| \times d_0}$ denotes the initial relation features. Our model is motivated by the first-order approximation of GCNs using Chebyshev polynomials \citep{Kipf2016}. Following \citet{gcn_srl}, we also allow the information in a directed edge to flow along both directions. Hence, we extend $\m{E}$ and $\m{R}$ with corresponding inverse edges and relations, i.e., 
\begin{equation*}
\m{E'} = \m{E} \cup \{(v,u,r^{-1})~|~ (u,v,r) \in \m{E}\} \cup \{(u, u, \top)~|~u \in \m{V})\} ,
\end{equation*}
and $\m{R}' = \m{R} \cup \m{R}_{inv} \cup \{\top\}$, where $\m{R}_{inv} =\{ r^{-1} \hspace{2pt} | \hspace{2pt} r \in \m{R} \}$ denotes the inverse relations and $\top$ indicates the self loop.

\begin{table*}[t]
	\centering
	\resizebox{\textwidth}{!}{
		\begin{tabular}{lccccc}
			\toprule
			\textbf{Methods} & Node		  & \multirow{2}{*}{Directions} &   \multirow{2}{*}{Relations}    		  & Relation& Number of \\
			& Embeddings				 &  			     &  & Embeddings &  Parameters\\
			\midrule
			GCN \cite{Kipf2016} 				& \checkmark & & & & $\mathcal{O}(Kd^2)$ \\
			Directed-GCN \cite{gcn_srl} 	& \checkmark & \checkmark & & &  $\mathcal{O}(Kd^2)$ \\
			Weighted-GCN \cite{sacn_paper} 			& \checkmark &  & \checkmark & &  $\mathcal{O}(Kd^2 + K|\m{R}|)$ \\
			Relational-GCN \cite{r_gcn} 					 & \checkmark & \checkmark & \checkmark & &  $\mathcal{O}(\m{B}Kd^2+ \m{B}K|\m{R}|)$ \\
			\midrule
			\method{} (Proposed Method)  & \checkmark & \checkmark & \checkmark & \checkmark & $\mathcal{O}(Kd^2 + \m{B}d + \m{B}|\m{R}|)$ \\
			\bottomrule
		\end{tabular}
	}
	
	\caption{\label{tbl:gcn_model_comp}Comparison of our proposed method, \method{} with other Graph Convolutional methods. Here, $K$ denotes the number of layers in the model, $d$ is the embedding dimension, $\m{B}$ represents the number of bases and $|\m{R}|$ indicates the total number of relations in the graph. Overall, \method{} is most comprehensive and is more parameter efficient than methods which encode relation and direction information.}
\end{table*}

\subsection{Relation-based Composition}
\label{sec:details_relation}
Unlike most of the existing methods which embed only nodes in the graph, \method{} learns a $d$-dimensional representation $\bm{h}_r \in \real{d}, \forall r \in \m{R}$ along with node embeddings $\bm{h}_v \in \real{d}, \forall v \in \m{V}$. Representing relations as vectors alleviates the problem of over-parameterization while applying GCNs on relational graphs. Further, it allows \method{} to exploit any available relation features $(\bmm{Z})$ as initial representations. To incorporate relation embeddings into the GCN formulation, we leverage the entity-relation composition operations used in Knowledge Graph embedding approaches \citep{transe,survey2016nickel}, which are of the form
\[
\bm{e}_o = \phi(\bm{e}_s, \bm{e}_r).
\]
Here, $\phi:\mathbb{R}^{d} \times \mathbb{R}^{d} \to \mathbb{R}^{d}$ is a composition operator, $s$, $r$, and $o$ denote subject, relation and object in the knowledge graph and $\bm{e}_{(\cdot)} \in \real{d}$ denotes their corresponding embeddings. In this paper, we restrict ourselves to non-parameterized operations like subtraction \citep{transe}, multiplication \citep{distmult} and circular-correlation \citep{hole}. However, \method{} can be extended to parameterized operations like Neural Tensor Networks (NTN) \citep{ntn_kg} and ConvE \citep{conve}. We defer their analysis as future work.

As we show in Section \ref{sec:results}, the choice of composition operation is important in deciding the quality of the learned embeddings. Hence, superior composition operations for Knowledge Graphs developed in future can be adopted to improve \method's performance further.

\subsection{\method{} Update Equation}

The GCN update equation (Eq. \ref{eq:dir_gcn}) defined in Section \ref{sec:back_relational_graph} can be re-written as
\[
\bm{h}_{v} = f \Bigg(\sum_{ (u,r) \in \mathcal{N}(v)}  \bm{W}_{r} \bm{h}_{u} \Bigg),
\]
where $\m{N}(v)$ is a set of immediate neighbors of $v$ for its outgoing edges. Since this formulation suffers from over-parameterization, in \method{} we perform composition ($\phi$) of a neighboring node $u$ with respect to its relation $r$ as defined above. This allows our model to be relation aware while being linear ($\mathcal{O}(|\m{R}|d)$) in the number of feature dimensions. Moreover, for treating original, inverse, and self edges differently, we define separate filters for each of them. The update equation of \method{} is given as:
\begin{equation}
\label{eq:node_update}
\bm{h}_{v} = f \Bigg(\sum_{ (u,r) \in \mathcal{N}(v)} \bm{W}_{\lambda(r)} \phi(\bm{x}_{u}, \bm{z}_r) \Bigg),
\end{equation}
where $\bm{x}_u, \bm{z}_r$ denotes initial features for node $u$ and relation $r$ respectively, $\bm{h}_{v}$ denotes the updated representation of node $v$, and $\bm{W}_{\lambda(r)} \in \real{d_1 \times d_0}$ is a relation-type specific parameter.  In \method{}, we use direction specific weights, i.e., $\lambda({r}) = \mathrm{dir}(r)$, given as:
\begin{equation}
\label{eqn:weight_def}
\bm{W}_{\mathrm{dir}(r)} =
\begin{cases} 
\bm{W}_O, & r \in \m{R}\\
\bm{W}_I, & r \in \m{R}_{inv}\\
\bm{W}_S, & r = \top  \ \text{\textit{(self-loop)}}
\end{cases}
\end{equation}

Further, in \method{}, after the node embedding update defined in Eq. \ref{eq:node_update}, the relation embeddings are also transformed as follows:
\begin{equation}
\label{eq:rel_share}
\bm{h}_r = \bm{W}_{\mathrm{rel}} \bm{z}_r   ,
\end{equation}
where $\bm{W}_{\mathrm{rel}} \in \real{d_1 \times d_0}$ is a learnable transformation matrix which projects all the relations to the same embedding space as nodes and allows them to be utilized in the next \method{} layer. In Table \ref{tbl:gcn_model_comp}, we present a contrast between \method{} and other existing methods in terms of their features and parameter complexity. 

\textbf{Scaling with Increasing Number of Relations} To ensure that \method{} scales with the increasing number of relations, we use a variant of the basis formulations proposed in \citet{r_gcn}. Instead of independently defining an embedding for each relation, they are expressed as a linear combination of a set of basis vectors. Formally, let $ \{\bm{v}_1, \bm{v}_2, ..., \bm{v}_{\m{B}} \}$ be a set of learnable basis vectors. Then, initial relation representation is given as:
\[
\bm{z}_r = \sum_{b=1}^{\mathcal{B}} \alpha_{{}_{br}} \bm{v}_b.
\]
Here, $\alpha_{{}_{br}} \in \real{}$ is relation and basis specific learnable scalar weight.  

\textbf{On Comparison with Relational-GCN} Note that this is different from the basis formulation in \cite{r_gcn}, where a separate set of basis matrices is defined for each GCN layer. In contrast, \method{} uses embedding vectors instead of matrices, and defines basis vectors only for the first layer. The later layers share the relations through transformations according to Equation \ref{eq:rel_share}. This makes our model more parameter efficient than Relational-GCN.


We can extend the formulation of Equation \ref{eq:node_update} to the case where we have $k$-stacked \method{} layers. Let $\bm{h}_v^{k+1}$ denote the representation of a node $v$ obtained after $k$ layers which is defined as
\begin{equation}
\label{eqn:main_upd}
\bm{h}_{v}^{k+1} = f \Bigg(\sum_{ (u,r) \in \mathcal{N}(v)} \bm{W}_{\lambda(r)}^k \phi(\bm{h}_{u}^k, \bm{h}_r^k) \Bigg) .
\end{equation}
Similarly, let $\bm{h}_r^{k+1}$ denote the representation of a relation $r$ after $k$ layers. Then,
\[
\bm{h}_r^{k+1} = \bm{W}_{\mathrm{rel}}^k \  \bm{h}_r^k.
\]
Here, $\bm{h}_v^0$ and $\bm{h}_r^0$ are the initial node ($\bm{x}_v$) and relation ($\bm{z}_r$) features respectively. \\ 


\begin{proposition}
\label{prop:reduction}
\method{} generalizes the following Graph Convolutional based methods: \textbf{Kipf-GCN} \citep{Kipf2016}, \textbf{Relational GCN} \citep{r_gcn}, \textbf{Directed GCN} \citep{gcn_srl}, and \textbf{Weighted GCN} \citep{sacn_paper}. 
\end{proposition}
\begin{proof}
	For Kipf-GCN, this can be trivially obtained by making weights $(\bm{W}_{\lambda(r)})$ and composition function $(\phi)$ relation agnostic in Equation \ref{eqn:main_upd}, i.e., $\bm{W}_{\lambda(r)} = \bm{W}$ and $\phi(\bm{h}_{u}, \bm{h}_r) = \bm{h}_u$. Similar reductions can be obtained for other methods as shown in Table \ref{tbl:reduction}.
\end{proof}

\begin{table}[t]
	\centering
	\small
	\begin{tabular}{llc}
		\toprule
		\multicolumn{1}{c}{\textbf{Methods}} & $\bm{W}^k_{\lambda(r)}$ & $\phi(\bm{h}^k_{u}, \bm{h}^k_r)$ \\ 
		\midrule
		Kipf-GCN \citep{Kipf2016} & $\bm{W}^k$ & $\bm{h}^k_u$ \vspace{0.5mm}\\ 
		Relational-GCN  \citep{r_gcn} & $\bm{W}^k_r$ & $\bm{h}^k_u$ \vspace{0.5mm}\\ 
		Directed-GCN \citep{gcn_srl} & $\bm{W}^k_{\mathrm{dir}(r)}$ & $\bm{h}^k_u$ \vspace{0.5mm}\\ 
		Weighted-GCN  \citep{sacn_paper} & $\bm{W}^k$ & $\alpha^k_{r} \bm{h}^k_u$\\ 
		\bottomrule
	\end{tabular}
	
	\caption{\label{tbl:reduction}Reduction of \method{} to several existing Graph Convolutional methods. Here, $\alpha^k_r$ is a relation specific scalar, $\bm{W}^k_r$ denotes a separate weight for each relation, and $\bm{W}^k_{\mathrm{dir}(r)}$ is as defined in Equation \ref{eqn:weight_def}. Please refer to Proposition \ref{prop:reduction} for more details.}
\end{table}

%% file: sections/experiments.tex
\section{Experimental Setup}


%


\subsection{Evaluation tasks}
\label{sec:exp_tasks}
In our experiments, we evaluate \method{} on the below-mentioned tasks.

\begin{itemize}[itemsep=2pt,parsep=0pt,partopsep=0pt,leftmargin=*,topsep=2pt]

\item  \textbf{Link Prediction} is the task of inferring missing facts based on the known facts in Knowledge Graphs. In our experiments, we utilize FB15k-237 \citep{toutanova} and WN18RR \citep{conve} datasets for evaluation. Following \cite{transe}, we use filtered setting for evaluation and report Mean Reciprocal Rank (MRR), Mean Rank (MR) and Hits@N.
	
\item  \textbf{Node Classification} is the task of predicting the labels of nodes in a graph based on node features and their connections. Similar to \cite{r_gcn}, we evaluate \method{} on MUTAG (Node) and AM \citep{rdf2vec} datasets.  

\item \textbf{Graph Classification}, where, given a set of graphs and their corresponding labels, the goal is to learn a representation for each graph which is fed to a classifier for prediction. We evaluate on 2 bioinformatics dataset: MUTAG (Graph) and PTC \citep{graph_datasets}.  

\end{itemize} 
A summary statistics of the datasets used is provided in Appendix \ref{sec:dataset_stats}
%

\begin{table*}[t]
	\centering
	\begin{small}
		\resizebox{\textwidth}{!}{
			\begin{tabular}{lccccccccccc}
				\toprule
				& \multicolumn{5}{c}{\textbf{FB15k-237}} && \multicolumn{5}{c}{\textbf{WN18RR}}\\ 
				\cmidrule(r){2-6}  \cmidrule(r){8-12} 
				& MRR & MR &H@10 &  H@3 & H@1 && MRR & MR & H@10 & H@3 & H@1 \\
				\midrule
				TransE \citep{transe}  & .294 & 357 & .465 & - & - && .226 & 3384 & .501 & - & - \\
				DistMult \citep{distmult}	& .241 & 254 & .419 & .263 & .155 && .43 & 5110 & .49  & .44 & .39 \\
				ComplEx	\citep{complex}		& .247 & 339 & .428 & .275 & .158 && .44  & 5261 & .51  & .46 & .41 \\
				R-GCN \citep{r_gcn}		& .248 & -   & .417 & & .151 && -    & -    & -    & & -  \\
				KBGAN \citep{kbgan}		& .278 & -   & .458 & & -    && .214 & -    & .472 & - & -\\
				ConvE \citep{conve}		& .325 & 244 & .501 & .356 & .237 &&  .43 & 4187 & .52  & .44 & .40  \\
				ConvKB \citep{convkb} & .243 & 311 & .421 & .371 & .155 && .249 & \textbf{3324} & .524 & .417 & .057 \\
				SACN \citep{sacn_paper} 		& .35  & -   & \textbf{.54} & \textbf{.39 }& .26  && .47  & -    & .54 & .48 & .43 \\
				HypER \citep{hyper} 	& .341 & 250 & .520 & .376 & .252 && .465 & 5798 & .522 & .477 & .436 \\
				RotatE \citep{rotate}		& .338 & \textbf{177} & .533 & .375 & .241 && .476 & 3340 & \textbf{.571} & .492 & .428 \\
				ConvR \citep{convr} 			& .350 & - & .528 & .385 & .261 && .475 & - & .537 & .489 & \textbf{.443} \\
				VR-GCN \citep{vrgcn} 	& .248 & - & .432 & .272 & .159 && - & - & - & - & - \\

				\midrule 
				\method{} (Proposed Method)		& \textbf{.355} & 197 & \textbf{.535} & \textbf{.390} & \textbf{.264} &&  \textbf{.479} & 3533 & .546  & \textbf{.494} & \textbf{.443}  \\
				\bottomrule
				\addlinespace
			\end{tabular}
		}
		\caption{\label{tbl:link_pred} \small Link prediction performance of \method{} and several recent models on FB15k-237 and WN18RR datasets. The results of all the baseline methods are taken directly from the previous papers ('-' indicates missing values). We find that \method{} outperforms all the existing methods on $4$ out of $5$ metrics on FB15k-237 and $3$ out of $5$ metrics on WN18RR. Please refer to Section \ref{sec:results_link} for more details. } \vspace{-3mm}
	\end{small}
\end{table*}

\subsection{Baselines}
\label{sec:exp_baselines}

Across all tasks, we compare against the following GCN methods for relational graphs: (1) Relational-GCN (\textbf{R-GCN}) \citep{r_gcn} which uses relation-specific weight matrices that are defined as a linear combinations of a set of basis matrices. (2) Directed-GCN (\textbf{D-GCN}) \citep{gcn_srl} has separate weight matrices for incoming edges, outgoing edges, and self-loops. It also has relation-specific biases. (3) Weighted-GCN (\textbf{W-GCN}) \citep{sacn_paper} assigns a learnable scalar weight to each relation and multiplies an incoming "message" by this weight. Apart from this, we also compare with several task-specific baselines mentioned below.

\textbf{Link prediction:} For evaluating \method{}, we compare against several non-neural and neural baselines: TransE \cite{transe}, DistMult \citep{distmult}, ComplEx \citep{complex}, R-GCN \citep{r_gcn}, KBGAN \citep{kbgan}, ConvE \citep{conve}, ConvKB \citep{convkb}, SACN \citep{sacn_paper}, HypER \citep{hyper}, RotatE \citep{rotate}, ConvR \citep{convr}, and VR-GCN \citep{vrgcn}.

\textbf{Node and Graph Classification:} For node classification, following \cite{r_gcn}, we compare with Feat \citep{feat}, WL \citep{wl}, and RDF2Vec \citep{rdf2vec}. Finally, for graph classification, we evaluate against \textsc{PachySAN} \citep{pachysan}, Deep Graph CNN (DGCNN)  \citep{dgcnn}, and Graph Isomorphism Network (GIN) \citep{gin}.

%% file: sections/results.tex
\section{Results}
\label{sec:results}


%

In this section, we attempt to answer the following questions.
\begin{itemize}[itemsep=1pt,topsep=2pt,parsep=0pt,partopsep=0pt]
	\item[Q1.] How does \method{} perform on link prediction compared to existing methods? (\ref{sec:results_link})
	\item[Q2.] What is the effect of using different GCN encoders and choice of the compositional operator in \method{} on link prediction performance? (\ref{sec:results_link})
	\item[Q3.] Does \method{} scale with the number of relations in the graph? (\ref{sec:results_basis})
	\item[Q4.] How does \method{} perform on node and graph classification tasks? (\ref{sec:res_node_nmt})
\end{itemize}

\vspace{-1mm}
\subsection{Performance Comparison on Link Prediction}
\label{sec:results_link}
\vspace{-1mm}
In this section, we evaluate the performance of \method{} and the baseline methods listed in Section \ref{sec:exp_baselines} on link prediction task. The results on \datafbn{} and \datawnn{} datasets are presented in Table \ref{tbl:link_pred}. The scores of baseline methods are taken directly from the previous papers \citep{rotate,kbgan,sacn_paper,hyper,convr,vrgcn}. However, for ConvKB, we generate the results using the corrected evaluation code\footnote{https://github.com/KnowledgeBaseCompleter/eval-ConvKB}. Overall, we find that \method{} outperforms all the existing methods in $4$ out of $5$ metrics on \datafbn{} and in $3$ out of $5$ metrics on \datawnn{} dataset. We note that the best performing baseline RotatE uses rotation operation in complex domain. The same operation can be utilized in a complex variant of our proposed method to improve its performance further. We defer this as future work.


\begin{table*}[!t]
	\centering
	\resizebox{\columnwidth}{!}{
		
		\begin{tabular}{m{13em}ccccccccccc}
			\toprule
			
			\multirow{1}{*}{\textbf{Scoring Function (=X)}} $\bm{\rightarrow}$ & \multicolumn{3}{c}{\textbf{TransE}} && \multicolumn{3}{c}{\textbf{DistMult}} && \multicolumn{3}{c}{\textbf{ConvE}} \\ 
			\cmidrule(r){2-4}  \cmidrule(r){6-8} \cmidrule(r){10-12}
			\textbf{Methods} $\bm{\downarrow}$ & MRR & MR & H@10 && MRR & MR & H@10 && MRR & MR & H@10 \\
			\midrule
			X	& 0.294	& 357	& 0.465	&& 0.241	& 354	& 0.419	&& 0.325	& 244	& 0.501 \\
			X + D-GCN & 0.299	& 351 & 0.469	&& 0.321	& 225 & 0.497	&&  0.344 & 200 & 0.524\\	
			X + R-GCN & 0.281	& 325	& 0.443	&& 0.324	& 230	& 0.499	&& 0.342	& 197	& 0.524 \\
			X + W-GCN & 0.267	& 1520	& 0.444	&& 0.324	& 229	& 0.504	&& 0.344	& 201	& 0.525 \\
			\midrule
			X + \method{} (Sub)	& 0.335	& \textbf{194}	& 0.514	&& 0.336	& 231	& 0.513	&& 0.352	& 199	& 0.530 \\
			X + \method{} (Mult)	& \textbf{0.337}	& 233	& 0.515	&& \textbf{0.338	}& \textbf{200}	& \textbf{0.518	}&& 0.353	& 216	& 0.532 \\
			X + \method{} (Corr)	& 0.336	& 214	& \textbf{0.518}	&& 0.335	& 227	& 0.514	&& \boxed{\textbf{0.355}} & 197	& \boxed{\textbf{0.535}} \\
			\midrule
			X + \method{} ($\m{B} = 50$) & 0.330 & 203 & 0.502 && 0.333 & 210 & 0.512 && 0.350 & \boxed{\textbf{193}} & 0.530 \\
			\bottomrule
		\end{tabular}
	}
	\caption{\label{tbl:link_pred_results} \small Performance on link prediction task evaluated on FB15k-237 dataset. X + M (Y) denotes that method M is used for obtaining entity (and relation) embeddings with X as the scoring function. In the case of \method{}, Y denotes the composition operator used. $\m{B}$ indicates the number of relational basis vectors used. Overall, we find that \method{} outperforms all the existing methods across different scoring functions. ConvE + \method{} (Corr) gives the best performance across all settings (highlighted using \boxed{\cdot}). Please refer to Section \ref{sec:results_link} for more details.}   \vspace{-3mm}
	
\end{table*}

\begin{figure*}[t]
	\centering
	\begin{minipage}{.485\textwidth}
		\centering
		\includegraphics[width=1.\linewidth]{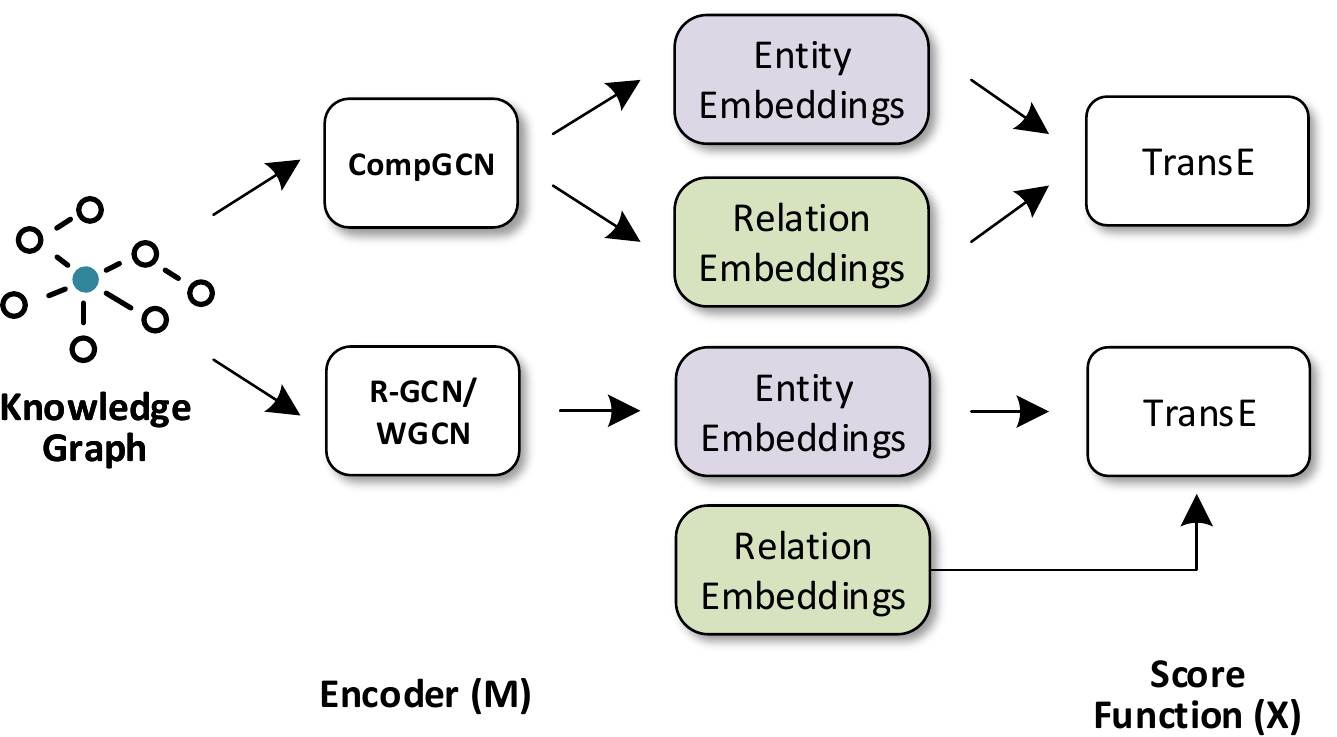}
		\caption{\label{fig:link_pred} \small Knowledge Graph link prediction with \method{} and other methods. \method{} generates both entity and relation embedding as opposed to just entity embeddings for other models. For more details, please refer to Section \ref{sec:results_gcn_encoder}}
	\end{minipage} \quad
	\begin{minipage}{0.47\textwidth}
		\centering
		\includegraphics[width=.99\linewidth]{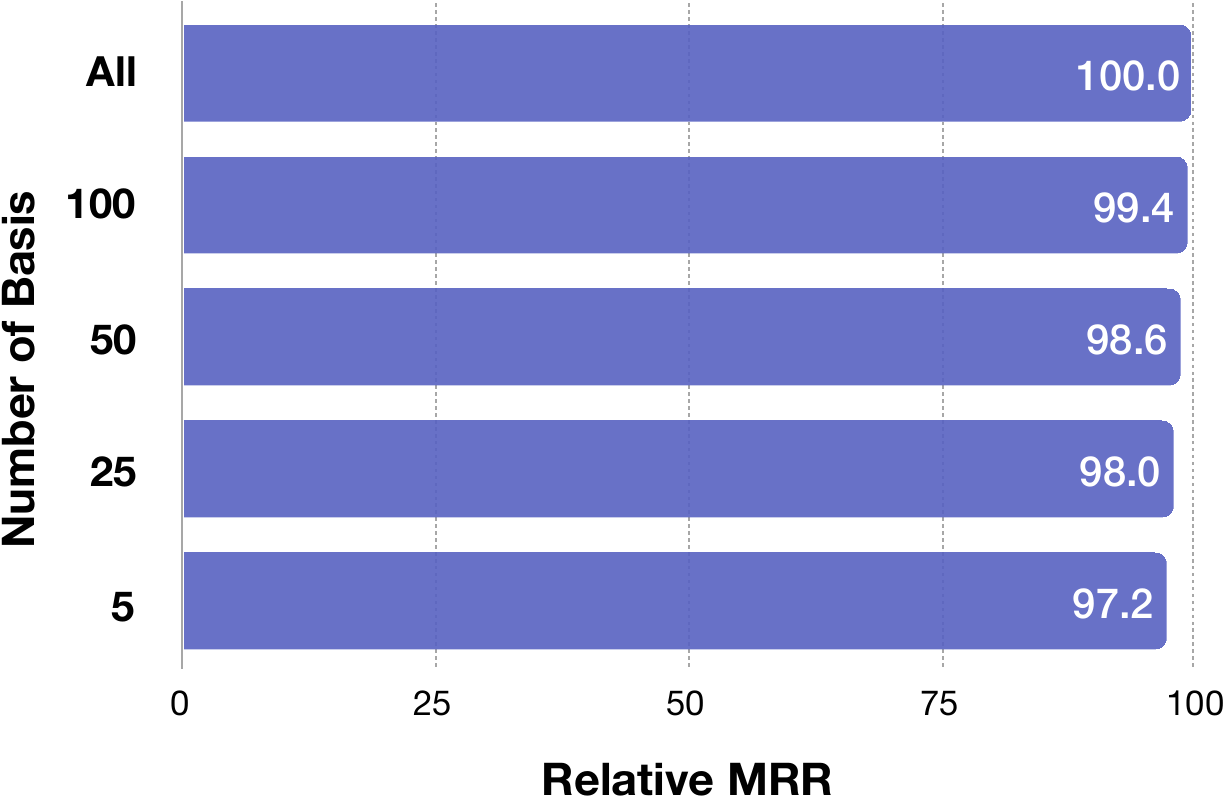}
		\caption{\label{fig:basis_plot} \small Performance of \method{} with different number of relation basis vectors on link prediction task. We report the relative change in MRR on FB15k-237 dataset. Overall, \method{} gives comparable performance even with limited parameters. Refer to Section \ref{sec:results_basis} for details.}
	\end{minipage} \vspace{-4mm}
\end{figure*}

\vspace{-1mm}
\subsection{Comparison of Different GCN Encoders on Link Prediction Performance}
\label{sec:results_gcn_encoder}
\vspace{-1mm}

Next, we evaluate the effect of using different GCN methods as an encoder along with a representative score function (shown in Figure \ref{fig:link_pred}) from each category: TransE (translational), DistMult (semantic-based), and ConvE (neural network-based).  In our results, \textbf{X + M (Y)}  denotes that method \textbf{M} is used for obtaining entity embeddings (and relation embeddings in the case of \method{}) with \textbf{X} as the score function as depicted in Figure \ref{fig:link_pred}. \textbf{Y} denotes the composition operator in the case of \method{}. We evaluate \method{} on three non-parametric composition operators inspired from TransE \citep{transe}, DistMult \citep{distmult}, and HolE \citep{hole} defined as
\begin{itemize}[itemsep=1pt,topsep=1pt,parsep=0pt,partopsep=0pt,leftmargin=5.5mm]
	\item \textbf{Subtraction (Sub):} $\phi(\bm{e}_s, \bm{e}_r) = \bm{e}_s - \bm{e}_r.$
	\item \textbf{Multiplication (Mult):} $\phi(\bm{e}_s, \bm{e}_r) = \bm{e}_s * \bm{e}_r.$
	\item \textbf{Circular-correlation (Corr):} $\phi(\bm{e}_s, \bm{e}_r) \text{=} \bm{e}_s \star \bm{e}_r$
\end{itemize}

The overall results are summarized in Table \ref{tbl:link_pred_results}. Similar to \citet{r_gcn}, we find that utilizing Graph Convolutional based method as encoder gives a substantial improvement in performance for most types of score functions. We observe that although all the baseline GCN methods lead to some degradation with TransE score function, no such behavior is observed for \method{}. On average, \method{} obtains around $6$\%, $4$\% and $3$\% relative increase in MRR with TransE, DistMult, and ConvE objective respectively compared to the best performing baseline. The superior performance of \method{} can be attributed to the fact that it learns both entity and relation embeddings jointly thus providing more expressive power in learned representations. Overall, we find that \method{} with ConvE (highlighted using \boxed{\cdot}) is the best performing method for link prediction.\footnote{We further analyze the best performing method for different relation categories in Appendix \ref{sec:results_rel_cat}.}

\noindent \textbf{Effect of composition Operator:} 
The results on link prediction with different composition operators are presented in Table \ref{tbl:link_pred_results}. 
We find that with DistMult score function, multiplication operator (Mult) gives the best performance while with ConvE, circular-correlation surpasses all other operators. Overall, we observe that more complex operators like circular-correlation outperform or perform comparably to simpler operators such as subtraction. 

\label{sec:scale}
\begin{figure*}[t]
	\centering
	\begin{minipage}{.485\textwidth}
		\centering
		\includegraphics[width=1.\linewidth]{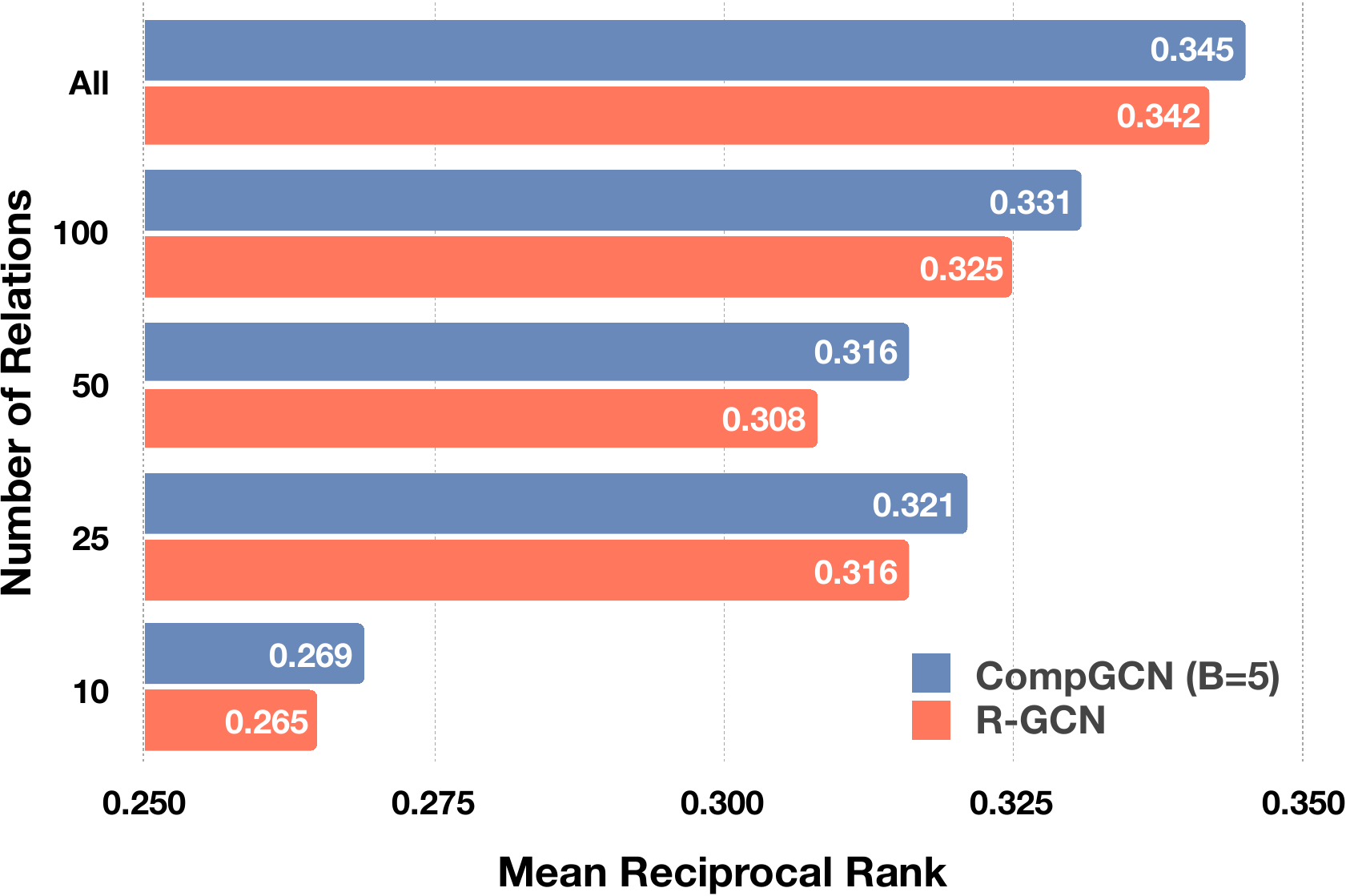}
		\caption{\label{fig:rgcn_compare} \small Comparison of \method{} ($\m{B} = 5$) with R-GCN for pruned versions of Fb15k-237 dataset containing different number of relations. \method{} with 5 relation basis vectors outperforms R-GCN across all setups. For more details, please refer to Section \ref{sec:scale}}
	\end{minipage} \quad
	\begin{minipage}{0.47\textwidth}
		\centering
		\includegraphics[width=.99\linewidth]{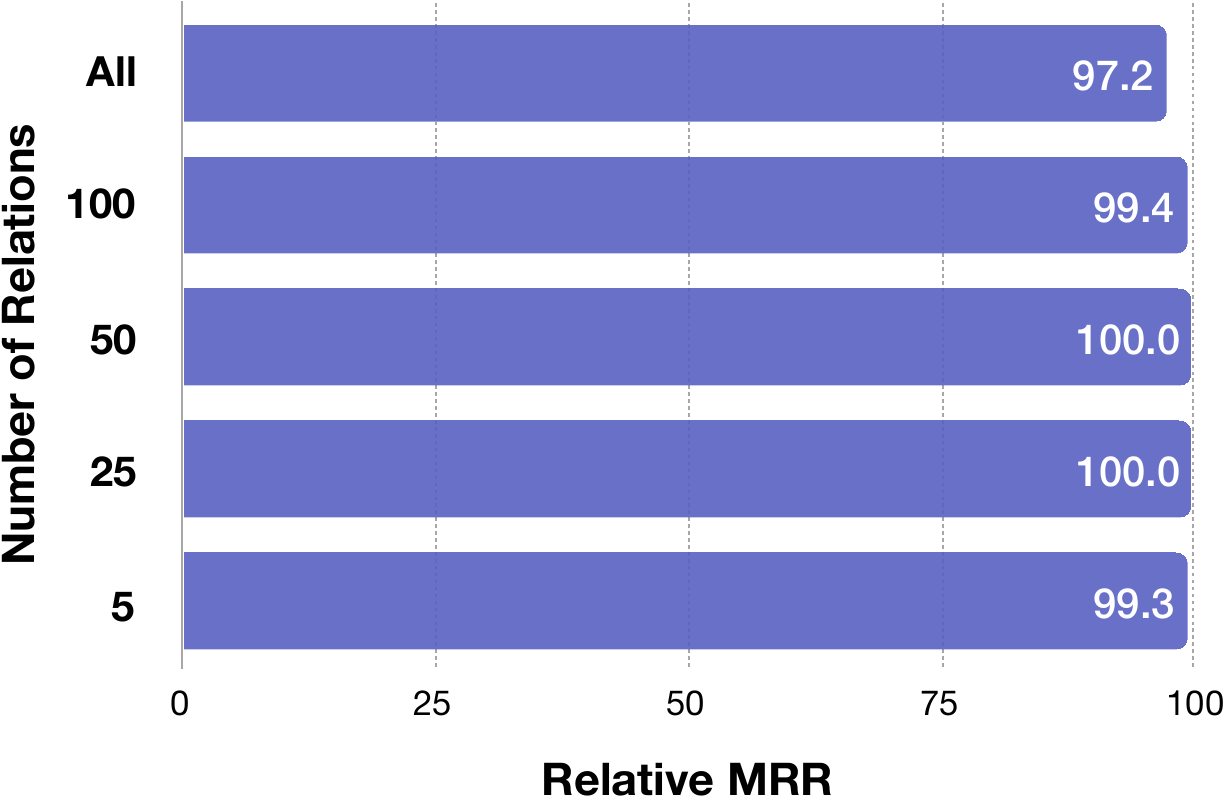}
		\caption{\label{fig:scale_plot} \small Performance of \method{} with different number of relations on link prediction task. We report the relative change in MRR on pruned versions of FB15k-237 dataset. Overall, \method{} gives comparable performance even with limited parameters. Refer to Section \ref{sec:scale} for details.}
	\end{minipage} \vspace{-4mm}
\end{figure*}

\vspace{-1mm}
\subsection{Scalability of \method{}} 
\label{sec:results_basis}
\vspace{-1mm}
In this section, we analyze the scalability of \method{} with varying numbers of relations and basis vectors. For analysis with changing number of relations, we create multiple subsets of FB15k-237 dataset by retaining triples corresponding to top-$m$ most frequent relations, where $m=\{10,25,50,100,237\}$. For all the experiments, we use our best performing model (ConvE + \method{} (Corr)).

\textbf{Effect of Varying Relation Basis Vectors:} Here, we analyze the performance of \method{} on changing the number of relation basis vectors ($\m{B}$) as defined in Section \ref{sec:details}. 
The results are summarized in Figure \ref{fig:basis_plot}. We find that our model performance improves with the increasing number of basis vectors. We note that with $\m{B}=100$, the performance of the model becomes comparable to the case where all relations have their individual embeddings. 
In Table \ref{tbl:link_pred_results}, we report the results for the best performing model across all score function with $\m{B}$ set to $50$. We note that the parameter-efficient variant also gives a comparable performance and outperforms the baselines in all settings. 

\textbf{Effect of Number of Relations:} Next, we report the relative performance of \method{} using $5$ relation basis vectors ($\m{B}=5$) against \method{}, which utilizes a separate vector for each relation in the dataset. The results are presented in Figure \ref{fig:scale_plot}. Overall, we find that across all different numbers of relations, \method{}, with a limited basis, gives comparable performance to the full model. The results show that a parameter-efficient variant of \method{} scales with the increasing number of relations.


\textbf{Comparison with R-GCN:} Here, we perform a comparison of a parameter-efficient variant of \method{} ($\m{B}=5$) against R-GCN on different number of relations. The results are depicted in Figure \ref{fig:rgcn_compare}. We observe that \method{} with limited parameters consistently outperforms R-GCN across all settings. Thus, \method{} is parameter-efficient and more effective at encoding multi-relational graphs than R-GCN.

\vspace{-1mm}
\subsection{Evaluation on Node and Graph Classification}
\label{sec:res_node_nmt}
\vspace{-1mm}
In this section, we evaluate \method{} on node and graph classification tasks on datasets as described in Section \ref{sec:exp_tasks}. The experimental results are presented in Table \ref{tbl:nodeclass_results}. For node classification task, we report accuracy on test split provided by \cite{node_class_splits}, whereas for graph classification, following \cite{graph_datasets} and \cite{gin}, we report the average and standard deviation of validation accuracies across the 10 folds cross-validation.  Overall, we find that \method{} outperforms all the baseline methods on node classification and gives a comparable performance on graph classification task. This demonstrates the effectiveness of incorporating relations using \method{} over the existing GCN based models. On node classification, compared to the best performing baseline, we obtain an average improvement of $3$\% across both datasets while on graph classification, we obtain an improvement of $3$\% on PTC dataset.

\begin{table}[!t]
	\centering
	\small
	\resizebox{\textwidth}{!}{
		\begin{tabular}{lcc}
			\toprule
			& \textbf{MUTAG (Node)} & \textbf{AM} \\
			\midrule
			Feat$^*$		& 77.9 & 66.7 \\ 
			WL$^*$				& 80.9 & 87.4 \\
			RDF2Vec$^*$	  & 67.2 & 88.3 \\
			R-GCN$^*$		& 73.2	&89.3 \\
			SynGCN	& 74.8 $\pm$ 5.5	& 86.2 $\pm$ 1.9	\\
			WGCN	& 77.9 $\pm$ 3.2	& 90.2 $\pm$ 0.9	\\
			\midrule
			\method{}	& \textbf{85.3 $\pm$ 1.2}	& \textbf{90.6 $\pm$ 0.2}	\\
			\bottomrule
			\addlinespace
		\end{tabular}
		\quad
		\begin{tabular}{lcc}
			\toprule
			& \textbf{MUTAG (Graph)} & \textbf{PTC} \\
			\midrule
			\sc{PachySAN}$^\dagger$			& \textbf{92.6 $\pm$ 4.2} & 60.0 $\pm$ 4.8 \\ 
			DGCNN$^\dagger$				& 85.8 & 58.6 \\
			GIN$^\dagger$	  	& 89.4 $\pm$ 4.7 & 64.6 $\pm$ 7.0 \\
			R-GCN & 82.3 $\pm$ 9.2	& 67.8 $\pm$ 13.2 \\
			SynGCN & 79.3 $\pm$ 10.3	& 69.4 $\pm$ 11.5 \\
			WGCN & 78.9 $\pm$ 12.0	& 67.3 $\pm$ 12.0 \\
			\midrule
			\method{}	& 89.0 $\pm$ 11.1	& \textbf{71.6 $\pm$ 12.0}	\\
			\bottomrule
			\addlinespace
		\end{tabular}
	}
	\caption{\label{tbl:nodeclass_results} \small Performance comparison on node classification (\textbf{Left}) and graph classification (\textbf{Right}) tasks.  $*$ and $\dagger$ indicate that results are directly taken from \cite{r_gcn} and \cite{gin} respectively. Overall, we find that \method{} either outperforms or performs comparably compared to the existing methods. Please refer to Section \ref{sec:res_node_nmt} for more details.} \vspace{-4mm}
\end{table}



%% file: sections/conclusion.tex
\vspace{-1mm}
\section{Conclusion}
\label{sec:conclusion}
\vspace{-1mm}

In this paper, we proposed \method{}, a novel Graph Convolutional based framework for multi-relational graphs which leverages a variety of composition operators from Knowledge Graph embedding techniques to jointly embed nodes and relations in a  graph. Our method generalizes several existing multi-relational GCN methods. Moreover, our method alleviates the problem of over-parameterization by sharing relation embeddings across layers and using basis decomposition. Through extensive experiments on knowledge graph link prediction, node classification, and graph classification tasks, we showed the effectiveness of \method{} over existing GCN based methods and demonstrated its scalability with increasing number of relations.

%% file: sections/appendix.tex
\subsection{Evaluation by Relation Category}
\label{sec:results_rel_cat}
In this section, we investigate the performance of \method{} on link prediction for different relation categories on FB15k-237 dataset. Following \citet{kg_relation_cat,rotate}, based on the average number of tails per head and heads per tail, we divide the relations into four categories: one-to-one, one-to-many, many-to-one and many-to-many. The results are summarized in Table \ref{tbl:results_rel_cat}. We observe that using GCN based encoders for obtaining entity and relation embeddings helps to improve performance on all types of relations. In the case of one-to-one relations, \method{} gives an average improvement of around $10$\% on MRR compared to the best performing baseline (ConvE + W-GCN). For one-to-many, many-to-one, and many-to-many the corresponding improvements are $10.5$\%, $7.5$\%, and $4$\%. These results show that \method{} is effective at handling both simple and complex relations.

\begin{table*}[!h]
	\centering
	\small
	\begin{tabular}{lm{3em}ccccccccccc}
		\toprule
		
		{} &  {} & \multicolumn{3}{c}{\textbf{ConvE}} && \multicolumn{3}{c}{\textbf{ConvE + W-GCN}} && \multicolumn{3}{c}{\textbf{ConvE + \method{} (Corr)}} \\
		\cmidrule(r){3-5}  \cmidrule(r){7-9} \cmidrule(r){11-13}
		{} & &   MRR & MR   & H@10  &&  MRR & MR   & H@10  &&  MRR & MR   & H@10   \\
		\midrule
		\multirow{4}{*}{\rotatebox[origin=c]{90}{Head Pred}} & 1-1 & 0.193	& 459	& 0.385	&& 0.422	& 238	& 0.547	&& \textbf{0.457	}& \textbf{150}	& \textbf{0.604} \\
		& 1-N & 0.068	& 922	& 0.116	&& 0.093	& 612	& 0.187	&& \textbf{0.112}	& \textbf{604}	& \textbf{0.190} \\
		& N-1 & 0.438	& 123	& 0.638	&& 0.454	& 101	& 0.647	&& \textbf{0.471}	& \textbf{99}	& \textbf{0.656} \\
		& N-N & 0.246	& 189	& 0.436	&& 0.261	& \textbf{169}	& 0.459	&& \textbf{0.275} & 179	& \textbf{0.474} \\
		\midrule
		\multirow{4}{*}{\rotatebox[origin=c]{90}{Tail Pred}} & 1-1 & 0.177	& 402	& 0.391	&& 0.406	& 319	& 0.531	&& \textbf{0.453}	& \textbf{193}	& \textbf{0.589} \\
		& 1-N & 0.756	& 66	& 0.867	&& 0.771	& 43	& 0.875	&& \textbf{0.779}	& \textbf{34}	& \textbf{0.885} \\
		& N-1 & 0.049	& 783	& 0.09	&& 0.068	& \textbf{747}	& 0.139	&& \textbf{0.076}	& 792	& \textbf{0.151} \\
		& N-N & 0.369	& 119	& 0.587	&& 0.385	& 107	& 0.607	&& \textbf{0.395	}& \textbf{102}	& \textbf{0.616} \\
		\bottomrule
	\end{tabular}
	\caption{\label{tbl:results_rel_cat}Results on link prediction by relation category on FB15k-237 dataset. Following \cite{kg_relation_cat}, the relations are divided into four categories: one-to-one (1-1), one-to-many (1-N), many-to-one (N-1), and many-to-many (N-N). We find that \method{} helps to improve performance on all types of relations compared to existing methods. Please refer to Section \ref{sec:results_rel_cat} for more details.}
\end{table*}

\subsection{Dataset Details}
\label{sec:dataset_stats}
In this section, we provide the details of the different datasets used in the experiments. For link prediction, we use the following two datasets: 

\begin{itemize}[itemsep=2pt,parsep=0pt,partopsep=0pt,leftmargin=*,topsep=2pt]
	\item \textbf{\datafbn{}} \citep{toutanova} is a pruned version of FB15k \citep{transe} dataset with inverse relations removed to prevent direct inference. 
	\item \textbf{\datawnn{}} \citep{conve}, similar to \datafbn{}, is a subset from WN18 \citep{transe} dataset which is derived from WordNet \citep{wordnet}. 
\end{itemize}

For node classification, similar to \cite{r_gcn}, we evaluate on the following two datasets: 
\begin{itemize}[itemsep=2pt,parsep=0pt,partopsep=0pt,leftmargin=*,topsep=2pt]
	\item \textbf{MUTAG (Node)} is a dataset from DL-Learner toolkit\footnote{http://www.dl-learner.org}. It contains relationship between complex molecules and the task is to identify whether a molecule is carcinogenic or not. 
	\item \textbf{AM} dataset contains relationship between different artifacts in Amsterdam Museum \citep{am_dataset}. The goal is to predict the category of a given artifact based on its links and other attributes. 
\end{itemize}

Finally, for graph classification, similar to \cite{gin}, we evaluate on the following datasets:
\begin{itemize}[itemsep=2pt,parsep=0pt,partopsep=0pt,leftmargin=*,topsep=2pt]
	\item \textbf{MUTAG (Graph)} \cite{mutag_graph} is a bioinformatics dataset of 188 mutagenic aromatic and nitro compounds. The graphs need to be categorized into two classes based on their mutagenic effect on a bacterium. 
	\item \textbf{PTC} \cite{ptc_dataset} is a dataset consisting of 344 chemical compounds which indicate carcinogenicity of male and female rats. The task is to label the graphs based on their carcinogenicity on rodents. 
\end{itemize}

A summary statistics of all the datasets used is presented in Table \ref{table:rgcn_data}.

\begin{table}[t]
	\centering
	\small
	\begin{tabular}{lcccccc}
		\toprule
		&  \multicolumn{2}{c}{\bf Link Prediction} & \multicolumn{2}{c}{\bf Node Classification} & \multicolumn{2}{c}{\bf Graph Classification}\\ 
		\cmidrule(r){2-3} \cmidrule(r){4-5} \cmidrule(r){6-7} 
		& \multicolumn{1}{c}{FB15k-237} & \multicolumn{1}{c}{WN18RR} & \multicolumn{1}{c}{MUTAG (Node)} &  \multicolumn{1}{c}{AM} & \multicolumn{1}{c}{MUTAG (Graph)} & \multicolumn{1}{c}{PTC}\\
		\midrule
		Graphs   & 1 & 1 & 1 & 1 & 188 & 344 \\
		Entities   & 14,541 & 40,943 & 23,644 & 1,666,764 & 17.9 (Avg) & 25.5 (Avg)\\
		Edges 	  & 310,116 & 93,003 & 74,227 & 5,988,321 & 39.6 (Avg) & 29.5 (Avg)\\
		Relations & 237 & 11 & 23 & 133 & 4 & 4 \\
		Classes   & - & - & 2 & 11 & 2 & 2\\
		\bottomrule
	\end{tabular}
	
	\caption{\label{table:rgcn_data}The details of the datasets used for node classification, link prediction, and graph classification tasks. Please refer to Section \ref{sec:exp_tasks} for more details.}
\end{table}

\subsection{Hyperparameters}
\label{sec:hyperparams}
Here, we present the implementation details for each task used for evaluation in the paper. For all the tasks, we used \method{} build on PyTorch geometric framework \citep{pytorch_geometric}. 

\noindent \textbf{Link Prediction:} For evaluation, $200$-dimensional embeddings for node and relation embeddings are used. For selecting the best model we perform a hyperparameter search using the validation data over the values listed in Table \ref{table:hyperparams}. For training link prediction models, we use the standard binary cross entropy loss with label smoothing \cite{conve}. 

\noindent \textbf{Node Classification:} Following \citet{r_gcn}, we use $10$\% training data as validation for selecting the best model for both the datasets. We restrict the number of hidden units to $32$. We use cross-entropy loss for training our model. 

\noindent \textbf{Graph Classification:} Similar to \cite{graph_datasets,gin}, we report the mean and standard deviation of validation accuracies across the 10 folds cross-validation. Cross-entropy loss is used for training the entire model. For obtaining the graph-level representation, we use simple averaging of embedding of all nodes as the readout function, i.e.,
\[
\bm{h}_{\m{G}} = \dfrac{1}{|\m{V}|}\sum_{v \in \m{V}} \bm{h}_v,
\]

where $\bm{h}_v$ is the learned node representation for node $v$ in the graph. 

For all the experiments, training is done using Adam optimizer \citep{adam_opt} and Xavier initialization \citep{xavier_init} is used for initializing parameters.
	
\begin{table}[h]
	\centering
	\small
	\begin{tabular}{ll}
		\toprule
		Hyperparameter                 & Values 					\\
		\midrule
		Number of GCN Layer ($K$) & \{1, 2, 3\}			\\
		Learning rate                  & \{0.001, 0.0001\}			\\
		Batch size                     & \{128, 256\}			\\
		Dropout  & \{0.0, 0.1, 0.2, 0.3\}		\\
		\bottomrule
	\end{tabular}
	\caption{Details of hyperparameters used for link prediction task. Please refer to Section \ref{sec:hyperparams} for more details.}
	\label{table:hyperparams}
\end{table}